\documentclass[a4paper,10pt]{article}
\usepackage[utf8]{inputenc}
\usepackage{natbib}
\usepackage{amsfonts}
\usepackage{amsmath}
\usepackage{bm}
\usepackage{bbm}
\usepackage[a4paper]{geometry}
\usepackage{algorithmicx}
\usepackage[noend]{algpseudocode}
\usepackage{algorithm}
\usepackage{graphicx}
\usepackage{lipsum}
\usepackage{inputenc}
\usepackage{multirow}
\usepackage{color}
\usepackage{amsthm}
\usepackage{authblk}
\usepackage{color}
 \usepackage[titletoc,toc,title]{appendix}

\newtheorem{prop}{Proposition}

\newtheorem{case}{Case}

\newtheorem{remark}{Remark}

\newtheorem{lemma}{Lemma}

\title{Adaptive regularization for Lasso models in the context of non-stationary data streams}

\author[1,2]{Ricardo Pio Monti}
\author[2]{Christoforos Anagnostopoulos}
\author[3]{Giovanni Montana}
\affil[1]{Gatsby Computational Neuroscience Unit, UCL, London, W1T 4JG, UK}
\affil[2]{Department of Mathematics, Imperial College London, London, SW7 2AZ, UK}
\affil[3]{WMG, University of Warwick, Coventry, CV4 7AL, UK }

\begin{document}

\date{}
\maketitle

\begin{abstract}

	Large scale, streaming datasets are ubiquitous in modern machine learning.
	Streaming algorithms must be scalable, amenable to  incremental 
	training and robust to the presence of non-stationarity. 
	In this work consider the problem of learning 
	$\ell_1$ regularized
	linear models in the context of streaming data.
	In particular, the focus of this work revolves around 
	how to select the regularization parameter when data arrives 
	sequentially and the underlying distribution is non-stationary (implying the choice of optimal regularization 
	parameter
	is itself time-varying).
	We propose a framework 
	through which to infer an adaptive regularization
	parameter.
	Our approach
	employs an $\ell_1$ penalty 
	constraint where the 
	corresponding sparsity parameter 
	is iteratively updated via stochastic gradient descent.
	This serves to reformulate the choice of regularization
	parameter in a principled framework for online learning.
	The proposed method is derived for linear regression and subsequently extended 
	to generalized linear models.  
	We validate our approach using simulated and real datasets and 
	present an application to a neuroimaging dataset.
	
\end{abstract}

\section{Introduction}
\label{intro}

We are interested in learning $\ell_1$ regularized  regression models in
the context of streaming, non-stationary data.
There has been significant research relating to the estimation of such models 
in a streaming data context \citep{bottou2010large, duchi2011adaptive}.
However a fundamental aspect that has been overlooked 
is the selection of the regularization parameter.
The choice of this parameter dictates the severity of the regularization penalty.
While the underlying optimization problem remains convex, distinct choices of such a parameter 
yield models with vastly different characteristics. 
This poses significant concerns from the perspective of model performance and interpretation.
It therefore follows that selecting such a parameter is an important problem that must be addressed in a 
data-driven manner.

Many solutions have been proposed through which to select the regularization parameter
in a non-streaming context.
For example, stability based approaches have been proposed in
the context of linear regression \citep{meinshausen2010stability}.
Other popular alternatives 
include 
cross-validation
and information 
theoretic techniques  \citep{hastie2015statistical}.
However, in a streaming setting such approaches are infeasible
due to the limited computational resources available. Moreover,
the statistical properties of the data may vary 
over time; a common manifestation being concept drift \citep{aggarwal2007data}. This
complicates the use of sub-sampling methods as the data can no longer 
be assumed to follow a stationary distribution. 
Furthermore, as we argue in this work, 
it is conceivable that the optimal choice of 
regularization parameter may itself vary over time.
It is also important to note that traditional approaches such 
as change-point detection cannot be employed
as there is no readily available pivotal quantity. 
It therefore follows that 
novel methodologies are required in order to 
tune regularization parameters in an online setting. 

Applications involving  
streaming datasets are abundant,
ranging from finance to cyber-security \citep{heard2010bayesian, gibberd2014high}
and neuroscience \citep{weiskopf2012real}.
In this work we are motivated by the latter application, where 
penalized regression models are often employed to 
decode statistical dependencies across 
spatially remote brain regions, referred to as functional connectivity \citep{smith2011network}.
A novel avenue for neuroscientific research involves the study of functional
connectivity in real-time \citep{weiskopf2012real}.
Such research faces challenges due to the 
non-stationary as well as potentially high
dimensional nature of neuroimaging data \citep{monti2014estimating}.
In order to address these challenges, 
many of the proposed methods to date have employed
fixed sparsity parameters. 
However, such choices are 
typically justified only by the methodological constraints associated with 
updating the regularization parameter, as opposed to 
for biological reasons.

In order to address these issues we propose 
a  framework through which to learn 
an adaptive sparsity parameter in an online fashion. 
The proposed framework, named Real-time Adaptive Penalization (RAP),
is capable of iteratively learning time-varying 
regularization parameters 
via the use of adaptive filtering.
Briefly, 
adaptive filtering methods are semi-parametric
methods which employ information from recent observations 
to tune a parameter of interest. 
In this manner, 
adaptive filtering 
methods are capable of handling temporal variation which cannot 
easily be modeled explicitly \citep{haykin2008adaptive}. 
The contributions of this work can be summarized  as follows:
\begin{enumerate}
	\item We propose and validate a framework through which to 
	tune a time-varying sparsity parameter for $\ell_1$ regularized 
	linear models in real-time.
	\item We 
	provide theoretical 
	insights regarding the properties and behavior of the proposed 
	method.\
	\item {\color{black} The proposed framework is subsequently extended to the context 
		regularized generalized linear models.}
	\item An empirical validation is provided using both synthetic and 
	real datasets together with an application to a neuroimaging dataset. 
\end{enumerate}

The remainder of this manuscript is organized as follows:
related work is discussed in Section \ref{sec:RelWork}.
We formally describe our problem in Section \ref{sec--Prelim} and
the proposed framework is introduced in Section \ref{sec--Methods}.
We provide empirical evidence based on real and simulated data in Section \ref{sec--Sims}.

\section{Related work}
\label{sec:RelWork}

Regularized methods have established themselves as popular and effective tools 
through which to handle high-dimensional data \citep{hastie2015statistical}.
Such methods employ regularization penalties as a mechanism through which to constraint 
the set of candidate solutions, often with the goal of enforcing specific properties such as parsimony.
In particular, 
$\ell_1$ regularization is widely employed as 
a convex approximation to the 
combinatorial problem of model selection. 

However, the introduction of an $\ell_1$ penalty requires 
the specification of the associated regularization parameter.
The task of tuning such a parameter has primarily been studied in 
the context of  non-streaming, stationary data.
Stability selection procedures, introduced by \cite{meinshausen2010stability},  
effectively look to by-pass the selection of 
a specific regularization parameter by instead 
fitting multiple models across sub-sampled data.
Variables are subsequently selected according to the 
proportion of all models in which they are present. 
In this manner, stability selection is able to provide important theoretical guarantees,
albeit 
while incurring an additional computational burden.
Other popular approaches involve the use of 
cross-validation 
or information theoretic techniques. 
However, such methods 
cannot be easily adapted to handle streaming data.

Online learning with the $\ell_1$ constraints has also been studied extensively and many computationally efficient algorithms
are available. 
A stochastic gradient descent algorithm is proposed by \cite{bottou2010large}.
More generally, 
online learning of regularized objective functions has been studied 
extensively by \cite{duchi2011adaptive}
who propose a general class of computationally efficient methods based on proximal gradient descent.
The aforementioned methods all constitute important advances in the study of sparse online learning algorithms.
However, a fundamental issue that has been overlooked corresponds to the selection of the 
regularization parameters. 
As such, current methodologies
are rooted on the assumption that the regularization parameter remains {fixed}. 
It follows that the regularization parameter may itself vary over time, yet
selecting such a parameter in a principled manner is non-trivial.
The focus of this work is to present and validate a framework through which to automatically 
select and update the regularization parameter in real-time. 
The framework presented in this work is 
therefore complementary and can be employed in conjunction with many 
of the preceding techniques. 
{\color{black}	
	In a similar spirit to the methods proposed in this manuscript, 
	\cite{garrigues2009homotopy} propose 
	a method for selecting the regularization parameter in the context of 
	sequential data but do not consider 
	non-stationary data, which is the explicit focus of this work.
	We further consider the extension to general linear models, leading to a 
	wider range of potential  applications.
}

More generally, the automatic selection of hyper-parameters has recently become an active topic in
machine learning \citep{shahriari2016taking}.
Interest in this topic
has been catalyzed by the success of deep learning algorithms, which typically involve many such hyper-parameters. 
Sequential model based optimization (SMBO)
methods such as Bayesian optimization 
employ a probabilistic surrogate 
to model the generalization performance of 
learning algorithms as samples from a Gaussian process \citep{shahriari2016taking}, leading to
expert level performance in many cases. 
It follows that such methods may be employed to tune regularization parameters in the context
of penalized linear regression models. However, there are several important differences
between the SMBO framework and the proposed framework. 
The most 
significant difference relates to the 
fact that the proposed framework employs gradient information in order to tune the regularization
parameter
while SMBO methods such as Bayesian optimization are rooted in the use of a  probabilistic 
surrogate
model. 
This allows the SMBO framework to be applied 
in a wide range of settings while the proposed framework 
focuses exclusively on Lasso regression models. However, as we describe in this work, the use of gradient 
information makes the 
RAP framework  ideally suited in the context of non-stationary, streaming data. This is in contrast to
SMBO techniques, which typically assume the data is stationary.

\section{Preliminaries}
\label{sec--Prelim}

In this section we 
introduce the necessary ingredients to derive the proposed framework. 
We begin 
formally defining the problem addressed in this work in Section \ref{sec--SetUp}.
Adaptive filtering methods are introduced in Section \ref{sec--AF}.

\subsection{Problem set-up}
\label{sec--SetUp}
In this work we are interested 
in streaming data problems. Here it is 
assumed that pairs $(X_t, y_t)$ 
arrive sequentially over time,
where $X_t \in \mathbb{R}^{p \times 1}$ corresponds to a $p$-dimensional vector of predictor variables and
$y_t$ is a univariate response.
The objective of this work is to learn time-varying linear regression models\footnote{We note that the proposed framework will be extended to Generalized Linear Models (GLMs) in Section \ref{sec--GLMextension}. For clarity we first formulate our approach in the context of linear regression. } from which to accurately 
predict future responses, $y_{t+1}$, from predictors, $X_{t+1}$.
An $\ell_1$ penalty, parameterized by $\lambda \in \mathbb{R}_{+}$, is introduced in order to 
encourage sparse solutions as well as to ensure the problem is well-posed 
from an optimization perspective. 
This corresponds to the Lasso model introduced by \cite{tibshirani1996regression}.
For a given choice of regularization parameter, $\lambda$, 
time-varying regression coefficients can be estimated by minimizing the following
convex objective function:
\begin{equation}
L_t (\beta, \lambda) = \sum_{i=1}^t w_i \left ( y_i - X_i^T \beta \right )^2 + \lambda ||\beta||_1,
\label{ConvexObjectiveLasso}
\end{equation}
where $w_i >0$ are weights indicating the importance given to past observations \citep{aggarwal2007data}. 
Typically, 
$w_i$ decay monotonically  in a manner which is proportional to the chronological proximity of the 
$i$th observation. 
For example, weights $w_i$ may be tuned using a fixed forgetting factor or a sliding window.

In a non-stationary context, the optimal estimates of regression 
parameters, $\hat \beta_t$, may vary over time. The same argument can be posed 
in terms of the selected regularization parameter, $\lambda$. 
For example, 
this may arise due to changes in the underlying 
sparsity or changes in the signal-to-noise ratio.
While there exists a wide range of methodologies through which 
to update regression coefficients in a streaming fashion, 
the choice of regularization parameter has been largely ignored.
As such, 
the primary objective of this work is to propose a framework through which to 
learn time-varying regularization parameter in real-time.
The proposed framework seeks to 
iteratively update the regularization parameter via stochastic gradient descent
and is therefore 
conceptually related to 
adaptive filtering theory
\citep{haykin2008adaptive}, which we introduce below.

\subsection{Adaptive filtering}
\label{sec--AF}

Filtering, as defined in \cite{haykin2008adaptive},
is the process through which information 
regarding a quantity of interest is assimilated
using data measured up to and including time $t$. 
In many real-time applications, the 
quantity of interest is assumed to
vary over time. 
The task of a filter therefore corresponds to 
effectively controlling the rate at which past information 
is discarded. 
Adaptive filtering methods provide an elegant method through 
which to handle a wide range of non-stationary behavior without having to explicitly model
the dynamic properties of the data stream. 

The simplest filtering methods discard information at a constant rate, for example
determined by a fixed forgetting factor. 
More sophisticated methods are able to exploit gradient information to 
determine the aforementioned rate. Such methods are said to be {adaptive}
as the rate at which information is discarded varies over time.
It follows that the benefits of adaptive methods are particularly notable 
in scenarios where the quantity of interest is highly non-stationary. 

To further motivate discussion, we briefly review filtering in the context of 
fixed forgetting factors for streaming linear regression. 
In such a scenario, 
it suffices to store summary statistics for the mean and sample covariance. For a given fixed
forgetting factor $r \in (0,1]$, the sample mean can be 
recursively estimated as follows:
\begin{equation}
\small
\bar X_t = \left (1 - \frac{1}{\omega_t} \right )\bar X_{t-1} + \frac{1}{\omega_t}X_t, 
\label{xbar}
\end{equation}
where $\omega_t$ is a normalizing constant defined as:
\begin{equation}
\omega_t = \sum_{i=1}^t r^{t-i} = r \cdot \omega_{t-1} + 1.
\end{equation}
Similarly, the sample covariance can be learned iteratively:
\begin{equation}
S_t = \left (1 - \frac{1}{\omega_t} \right ) S_{t-1} +  \frac{1}{\omega_t} (X_t - \bar X_t)^T (X_t - \bar X_t).
\label{sampCov}
\end{equation}

It is clear that the value of $r$ directly determines the adaptivity of a filter as well as its 
susceptibility to noise. However, in many practical scenarios the choice of $r$ presents a 
challenge as it assumes some knowledge about the \textit{degree} of non-stationarity of the 
system being modeled as well as an implicit assumption that 
this is constant  \citep{haykin2008adaptive}. Adaptive filtering 
methods address these issues by allowing  $r$ to be tuned online in a data-driven manner. 
This is achieved by 
quantifying the performance of current parameter estimates for new observations, $X_{t+1}$.
Throughout this work we denote such a measure by $C(X_{t+1})$.

A popular approach 
is to define $C(X_{t+1})$ to be the
residual error on unseen data \citep{haykin2008adaptive}. 
Then assuming $\frac{\partial C(X_{t+1})}{\partial r}$ can be efficiently calculated, 
our parameter of interest can be 
updated in a stochastic gradient descent framework:
\begin{equation}
r_{t+1} = r_t - \epsilon \left . \frac{\partial C(X_{t+1})}{\partial r} \right |_{r = r_t}
\end{equation}
where 
$\epsilon$ is a small step-size parameter which determines the learning rate.
The objective of this work therefore corresponds to 
extending adaptive filtering methods to the domain of learning
a time-varying regularization parameter for Lasso regression models.

\section{Methods }
\label{sec--Methods}

As noted previously, 
the choice of 
parameter $\lambda$ dictates the severity of 
the regularization penalty.
Different choices of $\lambda$ result in vastly different estimated 
models.
While several data-driven approaches are available for selecting $\lambda$
in an offline setting, such methods are typically not feasible 
for streaming data for two reasons.
First, limited computational resources pose a practical restriction.
Second, data streams are often non-stationary and rarely satisfy 
\textit{iid} assumptions required for methods based on the bootstrap \citep{aggarwal2007data}.
Moreover, it is important to note that traditional methods such as change point detection cannot be employed
due to the absence of a readily available pivotal quantity for $\lambda$.

We begin by outlining the RAP framework 
in the linear regression setting in 
Section \ref{sec--Framework}. 
Section \ref{sec--StreamLasso} outlines the resulting algorithm and 
computational considerations.
We derive some properties of the proposed framework in 
Section \ref{sec--FPconv}. Finally, in Section \ref{sec--GLMextension} we 
extend the RAP framework to the setting of GLMs.

\subsection{Proposed framework}
\label{sec--Framework}

We propose to learn a time-varying sparsity parameter in 
an adaptive filtering framework. 
This allows the proposed method to relegate the 
choice of sparsity parameter to the data. Moreover, by allowing $\lambda_t$ 
to vary over time the proposed method is able to naturally
accommodate datasets where the underlying sparsity may be non-stationary.

We define the empirical objective to be the look-ahead negative log-likelihood, defined as:
\begin{equation}
C_{t+1} = C(X_{t+1}, y_{t+1}) =  || y_{t+1} - X_{t+1}  \hat \beta_{t} (\lambda_{t})||_2^2,
\label{adapCost_LM}
\end{equation}
where we write $\hat \beta_t(\lambda_t)$ to emphasize the 
dependence of the estimated regression coefficients on the current 
value of the regularization parameter, $\lambda_t$.
Following Section \ref{sec--AF}, the regularization parameter can be iteratively 
updated as follows:
\begin{equation}
\label{lamSGD}
\lambda_{t+1} = G(\lambda_t) = \lambda_t - \epsilon \frac{\partial C_{t+1}}{\partial \lambda_t}.%
\end{equation}
We note that for convenience we write $\frac{\partial C_{t+1}}{\partial \lambda_t}$ to denote the 
derivate of $C_{t+1}$ with respect to $\lambda$ evaluated at $\lambda=\lambda_t$ (i.e, 
$\frac{\partial C_{t+1}}{\partial \lambda} |_{\lambda=\lambda_t}$). 
We note that $\lambda_t$ is bounded below by zero, in which case no 
regularization is applied, and above by  
$\lambda^{max}_t = \max_j \left \{ | \sum_{i=1}^t w_i y_i X_{i,j} | \right \},$ 
in which case all regression 
coefficients are  zero \citep{friedman2010regularization}.

The proposed framework requires only the specification of an initial sparsity parameter, $\lambda_0$,
together with a stepsize parameter, $\epsilon$. 
In this manner the proposed  framework
effectively replaces a fixed sparsity parameter with a stepsize parameter, $\epsilon$.
This is desirable as 
the choice of a fixed sparsity parameter is difficult to justify in the context of 
streaming, non-stationary data. Moreover, any choice of $\lambda$ is bound to be problem 
specific.
In comparison, we are able to interpret $\epsilon$ as a 
stepsize parameter in a stochastic gradient descent scheme. As a result, 
there are clear guidelines which can be followed when selecting $\epsilon$ \citep{bottou1998online}.

Once the regularization parameter has been updated, 
estimates for the corresponding regression coefficients can be obtained by minimizing 
$L_{t+1}(\beta, \lambda_{t+1})$, 
for which there is a wide literature available \citep{bottou2010large, duchi2011adaptive}.
The challenge in this work therefore corresponds to efficiently calculating the derivative in equation (\ref{lamSGD}).
Through the chain rule, this can be decomposed as:
\begin{equation}
\frac{\partial C_{t+1}}{\partial \lambda_t} = \frac{\partial C_{t+1}}{\partial \hat \beta_t} \cdot \frac{\partial \hat \beta_t}{\partial \lambda_t}.
\label{eqChainRule}
\end{equation}
The first term in equation (\ref{eqChainRule}) can be obtained by direct differentiation.
In the case of the second term, we leverage the results of \cite{efron2004least} 
and \cite{rosset2007piecewise} who demonstrate that the 
Lasso solution path is piecewise linear as a function of $\lambda$. 
By implication, $\frac{\partial \hat  \beta_t}{\partial \lambda_t}$
must be piecewise constant. Furthermore, there is a simple, closed-form solution for 
$\frac{\partial  \hat \beta_t}{\partial \lambda_t}$.

\begin{prop}
	\label{Prop1} [Adapted from \cite{rosset2007piecewise}]
	In the context of  $\ell_1$ penalized linear regression models, 
	the derivative $\frac{\partial \hat \beta_t}{\partial \lambda_t}$
	is piecewise constant and 
	can be obtained in closed form.
\end{prop}

\begin{proof}
	{\color{black}For any choice of regularization parameter, $\lambda$, we write 
		$\hat \beta_t( \lambda)$ to denote the minimizer of equation (\ref{ConvexObjectiveLasso}). 
		Recall that the objective, $L_t( \beta, \lambda)$, is non-smooth due to the 
		presence of the $\ell_1$ penalty. As a result, the sub-gradient 
		of $L_t( \beta, \lambda)$ must satisfy:
	}
	\begin{equation}
	\label{subGradEq}
	\nabla_{\beta} \left ( L_t(\beta, \lambda) \right )|_{\beta=\hat \beta_t (\lambda)} = - X_{1:t}^T W y_{1:t} + X_{1:t}^T W X_{1:t}\hat \beta_t(\lambda) + \lambda ~\mbox{sign} ( \hat \beta_t(\lambda)) \ni 0,
	\end{equation}
	{\color{black}
		where we $W$ is a diagonal matrix with elements $w_i$ and we write $X_{1:t}$ to denote
		a matrix where the $i$th row is $X_i$. 
		It is important to note that equation (\ref{subGradEq}) holds for 
		any choice of $\lambda$, however, the corresponding estimate 
		of regression coefficients, $\hat \beta_t( \lambda)$, 
		will necessarily change.
		Further, taking the derivative with respect to the regularization parameter $\lambda$
		yields:
		\begin{align*}
		\frac{\partial}{\partial \lambda}  \left (  \nabla_{\beta}  L_t(  \beta, \lambda)|_{\beta=\hat \beta_t (\lambda )}  \right ) 
		&= 0\\
		&= \frac{\partial \hat \beta_t (\lambda )}{\partial \lambda} \nabla_{_\beta}^2 L_t(\hat \beta_t (\lambda ), \lambda) + \mbox{sign} (\hat \beta_t (\lambda))\\
		&= \frac{\partial \hat \beta_t (\lambda )}{\partial \lambda} \left (X_{1:t}^T W X_{1:t} \right) + \mbox{sign}(\hat \beta_t (\lambda)).
		\label{longEq}
		\end{align*}
	}
	Rearranging 
	yields: 
	\begin{equation}
	\label{derivEq}
	\frac{\partial \hat \beta_t (\lambda )}{\partial \lambda} = -\left( X^T_{1:t} W X_{1:t} \right )^{-1} ~ \mbox{sign} ( \hat \beta_t (\lambda)) = -(S_t)^{-1} ~\mbox{sign} (\hat \beta_t (\lambda)).
	\end{equation}
\end{proof}

From Proposition 1 we have that the 
derivative, $\frac{\partial C_{t+1}}{\partial \lambda_t}$, can be computed in closed form.
{\color{black}
	Moreover, we note that the derivative in equation (\ref{derivEq}) is only 
	non-zero
	over the active set of regression coefficients, $\mathcal{A}_t = \{ i: (\hat \beta_t(\lambda_t))_i \neq 0 \}$,
	and zero elsewhere.} 
In practice we must therefore consider two scenarios:
\begin{itemize}
	\item the active set is non-empty (i.e., $\mathcal{A}_t \neq \emptyset$).
	In this case equation (\ref{derivEq}) is well-defined.
	\item  the active set is empty.
	In this case we 
	proceed to take a step in the direction of the most correlated predictor:
	$\hat j = \underset{j }{\operatorname{argmax}} \left \{ | \sum_{i=1}^t w_i y_i X_{i,j} | \right \}.$
	This is equivalent to the first step of the LARS algorithm \citep{efron2004least}. 
\end{itemize}

\subsection{Streaming Lasso regression}
\label{sec--StreamLasso}

At each iteration, a new pair $(X_{t+1},y_{t+1})$ is received and employed 
to update both the time-varying regularization parameter, $\lambda_t$,
as well as the corresponding 
estimate of regression coefficients, $\hat \beta_{t}(\lambda_{t})$.
The former involves computing the derivative $\frac{\partial C_{t+1}}{\partial \lambda_1}$
as outlined in Section \ref{sec--Framework}.
The latter involves solving a convex optimization problem 
which can be addressed in a variety of ways.
In this work we look to iteratively estimate regression coefficients using 
coordinate descent methods \citep{friedman2010regularization}. Such methods 
are easily amenable to streaming data and 
allow us to exploit
previous estimates as warm starts. 
In our experience, the use of warm starts leads to convergence within a handful of iterations. 
Pseudo-code detailing the proposed RAP framework is 
given in Algorithm \ref{alg:AdaptLam}.

\begin{algorithm}[th]
	\caption{\textbf{R}eal-time \textbf{A}daptive \textbf{P}enalization}
	\label{alg:AdaptLam}
	\begin{algorithmic}[1]
		\Require{$\epsilon \in \mathbbm{R}_+$ and $r \in (0,1]$}
		\For{$t \gets 1 \ldots t, \ldots  $}
		\State receive new $(X_{t+1}, y_{t+1})$
		\State compute $\frac{\partial \hat  \beta_t (\lambda )}{\partial \lambda_t}$ using equation (\ref{derivEq})
		\State set $\frac{\partial C_{t+1}}{\partial \lambda_t} = \frac{\partial C_{t+1}}{\partial \hat \beta_t (\lambda )} \frac{\partial \hat \beta_t (\lambda )}{\partial \lambda_t}$
		\State update $\lambda_{t+1} = \lambda_{t} - \epsilon \frac{\partial C_{t+1}}{\partial \lambda_{t}} $
		\State $\hat \beta_{t+1}(\lambda_{t+1}) =  \underset{\beta }{\operatorname{argmin}} \left \{ L_{t+1}(\beta, \lambda_{t+1}) \right \}$
		\EndFor
	\end{algorithmic}
\end{algorithm}

\subsubsection{Computational considerations}
With respect to the computational and memory demands, 
the 
major  expense incurred when calculating $ \frac{\partial \hat \beta_t (\lambda )}{\partial \lambda_t}$ 
involves inverting the sample covariance matrix. 
%
The need to compute and store the inverse of the sample covariance 
is undesirable in the context of high-dimensional data.
As a result,
the 
following approximation is also considered:
\begin{equation}
{\frac{\partial \hat \beta_t (\lambda )}{\partial \lambda_t}}  \approx -\left ( \mbox{diag} \left ( S_t \right)  \right )^{-1} ~ \mbox{sign}  \left ( \hat \beta_t (\lambda)  \right ).
\label{approximation} 
\end{equation}
Such approximations are 
frequently employed in streaming or large data applications \citep{duchi2011adaptive}.
The approximate update 
therefore has a time and memory complexity that is proportional to the 
cardinality of the active set,  $\mathcal{A}_t$.

\subsection{ Properties of the proposed framework}
\label{sec--FPconv}
{\color{black}
	In this section we study the properties of the proposed 
	framework. 
	We begin by
	showing that it is possible to divide the support of the 
	regularization parameter into
	a finite number of open subsets such that the 
	update rule is piecewise contractive within each subset. 
	We further show that any periodic behavior across adjacent subsets must also be 
	contractive. Unfortunately, as the support of $\lambda$ is divided into open subsets,
	this precludes the use of Banach's fixed point theorem. Nonetheless, the properties 
	detailed in this section provide important insights into the proposed 
	framework. 
}

We define $G(\lambda_{t}) = \lambda_t - \epsilon \frac{\partial C_{t+1}}{\partial \lambda_t}$
to be the  self-mapping defined on the support $\Lambda = [0, \lambda^{max}_t]$.
We study the behavior of iteratively applying the update rule 
$G(\lambda_t)$ for fixed new data pair $(X_{t+1}, y_{t+1})$. This corresponds to 
iteratively performing the gradient descent update to minimize
negative log-likelihood, 
$C_{t+1}$, for some fixed unseen pair, ($X_{t+1}, y_{t+1}$).
While the proposed algorithm is stochastic in the sense 
that distinct random samples, ($X_{t+1}, y_{t+1}$), are employed 
at each update step, the results presented below
provide reassuring 
insights.
We note that such non-stochastic results are often 
presented when studying online algorithms.
Throughout the remainder of this section we abuse notation and write $\lambda_{t+1} = G(\lambda_{t})$
to denote the result of applying the gradient update for $t$ iterations.
Finally, for any value of $\lambda \in \Lambda$, we
write $\mathcal{A}(\lambda)$ to denote the 
set of active 
regression coefficients. 

{\color{black}
	First, we 
	demonstrate that the support of the regularization parameter, $\Lambda$, 
	can be divided into a finite number of open subsets where $G$ is a 
	contraction mapping. 
	{\color{black} We then study periodic mappings across pairs of subsets 
		to show that such behavior is itself non-expansive.}

	\begin{remark}
		\label{LassoProperty1}
		The support 
		of the regularization parameter, $\Lambda$, can be divided into finitely many subsets, $\{ S_i \}$, such that the 
		active set within each subset is constant.  
	\end{remark}


	Remark \ref{LassoProperty1} is a widely used property of the Lasso
	and is related to the maximum
	number of iterations performed by the LARS algorithm 
	\citep{tibshirani2013lasso}. 

	\begin{lemma}
		\label{subsetLemma}
		The support of the regularization parameter can be divided into a finite number of 
		open subsets, 
		$\left \{ S_i \right \}$, where 
		$G$ is a contraction mapping.
	\end{lemma}
}
\begin{proof}
	{\color{black}
		From Remark 1 we note that the support of the regularization parameter 
		can be divided into a finite number of open subsets. It remains to show that 
		$G$ is a contraction within each subset. 
		
		We consider $\lambda_1, \lambda_2 \in S_i$ for some $i$.} 
	We assume without loss of generality that 
	$\lambda_1 > \lambda_2$. 
	We consider:
	\begin{equation}
	| G(\lambda_1) - G(\lambda_2) | = \left | \lambda_1 - \lambda_2 - \epsilon \left ( \frac{\partial C_{t+1}}{\partial \lambda_1} -  \frac{\partial C_{t+1}}{\partial \lambda_2}\right )\right |.
	\end{equation}
	Our objective is to show that 
	$\frac{\partial C_{t+1}}{\partial \lambda_1} -  \frac{\partial C_{t+1}}{\partial \lambda_2}>0$,
	thereby showing that $G$ is a contraction for suitably chosen $\epsilon$.
	The gradient with respect to regularization parameter $\lambda$ is defined as:
	$$\frac{\partial C_{t+1}}{\partial \lambda}  = \left ( y_{t+1} -X_{t+1} \hat \beta_t(\lambda)   \right )^T X_{t+1}^T \left( S_{t}\right)^{-1} \mbox{sign} \left( \hat \beta_t (\lambda) \right) $$
	{\color{black} Furthermore,  we have that:}
	\begin{equation*}
	\footnotesize
	\begin{aligned}
	\frac{\partial C_{t+1}}{\partial \lambda_1} - \frac{\partial C_{t+1}}{\partial \lambda_2}&=  \underbrace{\sum_{i \in \mathcal{A}(\lambda_1  ) \cap \mathcal{A}(\lambda_2)} \left [ \left ( \hat \beta_t(\lambda_2) - \hat \beta_t(\lambda_1) \right )^T \left ( X_{t+1}^T X_{t+1} \right ) \left ( S_{t} \right )^{-1} \mbox{sign}(\hat \beta_t(\lambda_1)) \right ]_i}_{A_1} \\
	&\underbrace{- \sum_{i \in \mathcal{A}(\lambda_2) \backslash \mathcal{A}(\lambda_1)} \left [ \left ( y_{t+1} - X_{t+1} \hat  \beta_t(\lambda_2) \right)^T X_{t+1}^T \left( S_{t}\right)^{-1} \mbox{sign} (\hat  \beta_t(\lambda_2)) \right ]_i}_{A_2}
	\end{aligned}
	\end{equation*}
	{\color{black} and we note that the latter term will be zero whenever 
		$\mathcal{A}(\lambda_1) \backslash \mathcal{A}(\lambda_2) = \emptyset$. 
		This holds by construction in our case as 
		$\lambda_1, \lambda_2 \in S_i$. 
	}
	Moreover, the term $A_1$ will always be greater than or equal to zero. This follows from
	the fact that $A_1 = g(\lambda_1) - g(\lambda_2)$ where
	\begin{equation*}
	\begin{aligned}
	g(\lambda) &= - \left( \hat \beta_t(\lambda)^T \left(X_{ t+1}^T X_{t+1}\right) \left ( S_{t}\right )^{-1} \mbox{sign}( \hat \beta_t(\lambda))\right)  \\
	&=\hat \beta_t(\lambda)^T \left(X_{t+1}^T X_{t+1}\right) \frac{\partial \hat \beta_t(\lambda)}{\partial \lambda}.
	\end{aligned}
	\end{equation*}
	Therefore, we have that:
	\begin{equation}
	\frac{\partial g (\lambda)}{ \partial \lambda} = 
	\left(\frac{\partial \hat \beta_t(\lambda)}{\partial \lambda}^T \left(X_{t+1}^T X_{t+1}\right) \frac{\partial \hat  \beta_t(\lambda)}{\partial \lambda}  \right) \geq 0,
	\end{equation}
	{\color{black}
		due to the positive semi-definite nature of $X_{t+1}^T X_{t+1}$ and the fact that the second derivative of $\hat \beta_t (\lambda)$ with respect to $\lambda$ is zero. } This
	indicates that $g(\lambda)$ is a monotone, non-decreasing function in $\lambda$ within the subset $S_i$. 
	As a result, we have that the mapping $G$ will be contraction on the open subset $S_i$.
	These subsets correspond to the regions where the 
	support of the Lasso solution is constant,
	thus implying that $A_2$ is zero. 
	
\end{proof}

By Lemma \ref{subsetLemma}, we have that $| G(\lambda_1) -  G(\lambda_2)| < |\lambda_1 - \lambda_2|$ for 
all $\lambda_1, \lambda_2 \in S_i$.
{\color{black}The following Lemma demonstrates that 
	alternating periodic behavior across two 
	adjacent subsets, $S_j$ and $S_{j-1}$, is also contractive. }

\begin{lemma}
	If periodic behavior occurs across two adjacent subsets, 
	then this must be a contraction. 
	\label{cyclicLemma}
\end{lemma}
\begin{proof}
	We consider periodic behavior of the form:
	\begin{equation}  G(\lambda_{t}) \in  \begin{cases}
	S_j       & \quad \mbox{if } t \mbox{ is even}\\
	S_{j-1}  & \quad \mbox{if } t \mbox{ is odd}\\
	\end{cases} 
	\label{cyclicCases}
	\end{equation}
	We consider two subsets which we label $S_1$ and $S_2$. Without loss of generality we assume
	that $S_1 > S_2$ in the sense that $\lambda_1 > \lambda_2$ for all $\lambda_1 \in S_1$ and $\lambda_2 \in S_2$.
	We consider the periodic behavior described in equation (\ref{cyclicCases}).
	
	Therefore, at an odd iteration the 
	gradient update maps from $S_2$ into $S_1$.
	Thus we have $\lambda_{t} = G(\lambda_{t-1}) > \lambda_{t-1}$ by construction.
	This implies that 
	$\frac{\partial C_t}{\partial \lambda_{t-1}} < 0$. 
	Conversely, in every even iteration the gradient update maps from $S_1$ into $S_2$, implying that
	$\lambda_{t} = G(\lambda_{t-1}) < \lambda_{t-1}$.
	This in turn implies that $\frac{\partial C_t}{\partial \lambda_{t-1}} > 0$. 
	
	As a result, we have that for any $\lambda_1 \in S_1$ and $\lambda_2 \in S_2$:
	$$ \frac{\partial C_t}{\partial \lambda_1} - \frac{\partial C_t}{\partial \lambda_2} > 0$$
	indicating that cyclic mapping must be contractions. 
\end{proof}

We note that the aforementioned results also hold 
when either the exact or approximate gradient as well as 
when multiple unseen samples $\{(X_{i}, y_i): i=1,\ldots, T\}$ are
employed (as in the case of mini-batch updates).

{\color{black}
	\subsection{Extension to Generalized Linear Models}
	\label{sec--GLMextension}
	While the preceding sections focused on 
	linear regression, we now extend the proposed framework to 
	a wider class of 
	GLM models. 
	As such, we assume that observations $y_t$ follow an exponential
	family distribution such that 
	$\mathbb{E}[y_t] = \mu_t$
	and $ \mbox{Var}(y_t)=V_t$. 
	In the context of GLMs, it is assumed that a
	(potentially non-linear) link function is employed to relate
	the mean, $\mu_t$, to a linear combination of 
	predictors:
	\begin{equation}
	\eta_t = g( \mu_t ) = X_t^T \beta_{t-1}.
	\end{equation}
	We note that when $y_t$ is assumed to follow a Gaussian distribution 
	we recover linear regression as described in Section \ref{sec--Framework}. 
	Conversely, if $y_t$ follows a Binomial distribution we obtain 
	streaming logistic regression. 
	The log-likelihood of an observed response, $y_t$, can be expressed as \citep{mccullagh1989generalized}:
	\begin{equation}
	\label{GLMloglike}
	l(y_t; \theta_t) = \frac{y_t \theta_t - b(\theta_t)}{a(\phi)} + c(y_t, \phi),
	\end{equation}
	where $a(\cdot), b(\cdot)$ and $c(\cdot)$ are functions which vary according to the distribution of the response
	and $\theta_t = \theta(\beta_t)$ is the corresponding canonical parameter.
	Throughout this work it is assumed that 
	the dispersion parameter, $\phi$, is known and fixed.
	
	Analogously to equation (\ref{ConvexObjectiveLasso}), 
	we estimate $\ell_1$ regularized 
	regression coefficients by minimizing the re-weighted negative log-likelihood objective:
	\begin{equation}
	\label{l1Objective}
	L_t(\beta, \lambda) =  - \sum_{i=1}^t w_i \left [ y_i ~\theta (\beta_i) - b \{ \theta (\beta_i) \}  \right ] + \lambda || \beta ||_1,
	\end{equation}
	where $w_i$ are weights as before. In the remainder of this manuscript we
	focus on two popular cases, detailed below, but we note that 
	the proposed framework can be employed in a much wider range of settings. 

	\begin{case}
		Normal linear regression. 
		\normalfont In the case of linear regression we have that $g(\cdot)$ is the identity
		such that $\theta_{t+1} = X_{t+1}^T \hat \beta_t (\lambda_t)$
		and $C_{t+1}$ is defined as in Section \ref{sec--Framework}.
	\end{case}
	
	\begin{case}
		Logistic regression. 
		\normalfont In this case we have that $g(\cdot)$ is the logistic function. As before
		$\theta_{t+1} = X_{t+1}^T \hat \beta_t (\lambda_t)$ and the negative log-likelihood is defined as:
		$$C_{t+1} = C(X_{t+1}, y_{t+1}) = -y_{t+1} X_{t+1} \hat \beta_t (\lambda_t) + \mbox{log} \left ( 1 + e^{X_{t+1}^T \hat \beta_t (\lambda_t)} \right ). $$
	\end{case}
	
	\begin{prop} 
		\label{prop2} [Adapted from \citep{park2007l1}]
		In the context of $\ell_1$ penalized GLM models, the derivative $\frac{\partial \hat \beta_t (\lambda )}{\partial \lambda_t}$ is also available in closed form as follows:
		\begin{equation}
		\frac{\partial \hat \beta_t (\lambda )}{\partial \lambda_t} = -\left ( X_{1:t}^T W X_{1:t}  \right )^{-1} \mbox{ \normalfont sign} (\hat \beta_t (\lambda))
		\end{equation}
		where $W$ is a diagonal matrix with entries $w_i V_i^{-1} \left ( \frac{\partial \mu_i }{\partial \eta_i}\right )$. 
	\end{prop}
	
	\begin{proof}
		The proof is closely related to that of Proposition \ref{Prop1}.
		A full derivation is 
		provided in Appendix \ref{AppendixA}.
	\end{proof}

	
	\begin{remark}
		We note that applying the RAP framework in the context of GLMs 
		requires only minor modifications from the procedure  
		detailed Algorithm \ref{alg:AdaptLam}. 
	\end{remark}
	
	\begin{remark}
		Unfortunately, the paths of regression coefficients within regularized GLM models are not piece-wise linear \citep{park2007l1}. As such, 
		the results of Section \ref{sec--FPconv} cannot easily be extended to 
		include regularized GLM models.
	\end{remark}
	
}
\section{Empirical results}
\label{sec--Sims}

\label{sec--SimStudy}

In this section we empirically demonstrate the capabilities of the 
proposed framework via a series of simulations. 
We begin by considering the performance of the RAP algorithm in the context of 
stationary data. This simulation serves to demonstrate that the 
proposed method is capable of accurately tracking the regularization parameter. 
We then study the performance of RAP algorithm in the context of non-stationary data.
Throughout this simulation study the RAP algorithm is benchmarked against two offline methodologies:
cross-validation and SMBO. 
In the context of SMBO methods, we study the performance against Bayesian optimization methods. Here a 
Gaussian process with a square exponential kernel was employed as a surrogate model together with 
the expected improvement acquisition function.


\subsection{Simulation settings}
\label{sec--SimSettings}

In order to thoroughly test the performance of the RAP algorithm, 
we look to generate synthetic data were 
we are able to control both the underlying structure as well as the dimensionality of the
data. 
In this work, 
the covariates $X_t$ were 
generated 
according to a multivariate Gaussian distribution with a block covariance structure. This introduced
significant correlations across covariates, thereby increasing the difficultly of the 
regression task. 
Formally, the data simulation process followed that described by \cite{mcwilliams2014loco}. 
This involved sampling each covariate as follows:
$$ X_t \sim \mathcal{N}(0, \Sigma),$$
where $\Sigma \in \mathbb{R}^{p \times p}$ is a block diagonal matrix consisting of five equally sized blocks. 
Within each block, the off-diagonal entries were fixed at $0.8$, while the diagonal entries were fixed to be one. 
Having generated covariates, $X_t$, a sparse vector of regression coefficients, $\beta$, was simulated.
This involved randomly selecting a proportion, $\rho$, 
of coefficients and randomly generating their values according to a standard Gaussian distribution. 
All remaining coefficients were set to zero. 
{\color{black}
	Given simulated covariates, $X_t$, and a vector of sparse regression coefficients 
	$\beta$,  the response was simulated according to an exponential family distribution
	with mean parameter $\mu_t = g^{-1} (X_t \beta )$. In this manner, data was 
	generated from both a Gaussian as well as Binomial distributions. 
	In case of the former, we therefore have that $ y_t \sim \mathcal{N}( X_t^T \beta, 1)$,
	while in the case of logistic regression we have 
	that $y_t$ follows a Bernoulli distribution with 
	mean $\sigma (X_t^T \beta )$ where 
	$\sigma(\cdot)$ denotes the logistic function. 
}

In this manner, it is possible to generate piecewise stationary data, $\{ (y_t, X_t): t=1, \ldots, T\}$.
When studying the performance of the RAP algorithm in the context of stationary data, it sufficed to simulate
one such dataset. In order to quantify performance in the context of non-stationary data, we concatenate 
multiple piece-wise stationary datasets. This results in datasets with abrupt changes. 
We note that in the non-stationary setting the block structure was randomly permuted at each iteration 
to avoid covariates sharing the same set of correlated variables. 

\subsection{Performance metrics}
\label{sec--PerfMetrics}

In order to assess the performance of the RAP algorithm we consider various  metrics. 
In the context of stationary data, our primary objective is to demonstrate that the 
proposed method is capable of tracking the regularization parameter when benchmarked against 
traditional methods such as cross-validation. As a result, 
we consider the difference in $\ell_1$ norms of the 
regression model estimated by each 
algorithm. This is defined as:
\begin{equation}
\Delta = ||\beta(\lambda^{CV})||_1 - || \beta(\lambda^{RAP})||_1,
\label{l1DiffEq}
\end{equation}
where we write $\lambda^{CV}$ and $\lambda^{RAP}$ to denote the regularization parameters
selected by cross-validation and RAP algorithms respectively.
We choose to employ the $\ell_1$ norm (as opposed to directly considering the sparsity parameter, $\lambda$)
as there is a one-to-one relationship between $\lambda$ and the $\ell_1$ norm. This serves to bypass any 
potential issues arising from scaling or other idiosyncrasies.

{\color{black}  
	In the context of non-stationary data we are interested in two additional metrics. 
	The first corresponds to the negative log-likelihood 
	of each new unseen observations, $C_{t+1}$, initially defined 
	in equation (\ref{adapCost_LM}).
}
Secondly, we also consider the correct
recovery of the sparse support of $\beta_t$. In this context, we
treat the recovery of the support of $\beta_t$ as a binary classification problem
and quantify the performance using the $F$ score; defined as the 
harmonic mean between the precision and recall of a 
classification algorithm. 

\subsection{Stationary data}
\label{sim--StationaryDataSim}

{\color{black}
	We begin by demonstrating that the RAP framework is capable of accurately 
	tracking the regularization parameter in the context of stationary data. 	
}
In particular, we study the performance of the RAP algorithm as
the dimensionality of regression coefficients, $p$, increases. 

Data was generated as described in Section \ref{sec--SimSettings} and the dimensionality
of the covariates, $X_t$, was varied from $p=10$ through to $p=100$. 
For each value of $p$, 
datasets consisting of $n=300$ observations where randomly generated. 
The regularization parameter was first estimated using $K=10$ fold cross-validation.
The RAP algorithm was subsequently employed and the 
difference in $\ell_1$ norm, defined in equation (\ref{l1DiffEq}), was then computed. 
In case the of the RAP algorithm, each observation was studied once in a streaming fashion. The initial choice for the regularization parameter, $\lambda_0$, was 
randomly sampled from a uniform distribution, $\mathcal{U}[0,1]$. 
{\color{black}
	Both normal linear  and logistic regression were studied in this manner.}

\begin{figure}[t]
	\includegraphics[width=0.8\textwidth]{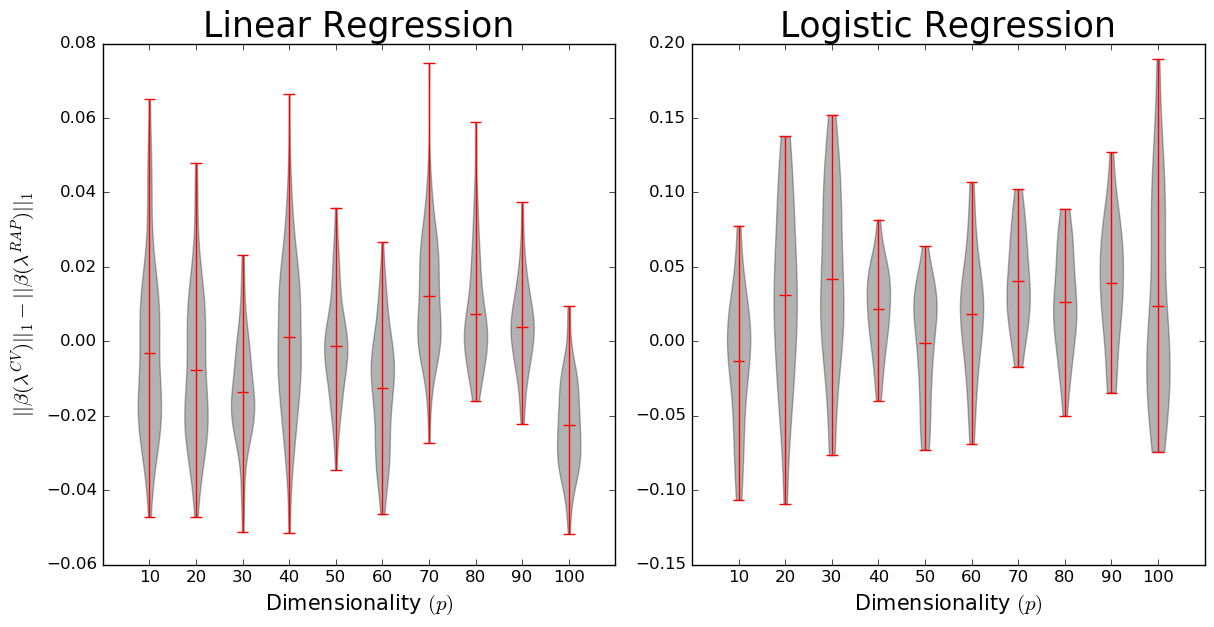}
	\centering
	\caption{Violin plots visualizing the difference in selected regularization parameters
		as a function of the {\color{black} dimensionality, $p$, for linear (left) and 
			logistic regression (right). 
			We note that the difference in estimated $\ell_1$ norms is 
			both small and centered around the origin, indicating the absence of large
			systematic bias. Note the difference in y-axis across panels. 
		}
	}
	\label{Sim1bFig}
\end{figure}

{\color{black}
	The difference in selected regularization parameters over $N=500$
	simulations is visualized in Figure \ref{Sim1bFig}. 
	It is reassuring to note that, 
	for both linear and logistic regression,
	the differences are both small in magnitude as well as centered around the 
	origin. This serves to indicate the absence of a large systematic bias. 
	However, 
	we note that there is  higher 
	variance in the context of logistic regression. 
}

\subsection{Non-stationary data}
\label{sim--NonStationaryDataSim}

While Section \ref{sim--StationaryDataSim} provided empirical evidence demonstrating that the 
RAP framework can be effectively employed to track regularization parameters in a stationary
setting, we are ultimately interested in 
streaming, non-stationary datasets. 
As a result, in this simulation we study 
the performance of the proposed framework in the context of non-stationary data. 
{\color{black}
	As in Section \ref{sim--StationaryDataSim} we study the properties of the RAP algorithm in the context of
	linear and logistic regression.}

While 
there are a multitude of methods through which to simulate 
non-stationary data, in this simulation study we chose to generate data with 
piece-wise stationary covariance structure. 
As a result, 
the underlying covariance alternated between two regimes: 
a sparse regime where the response was driven by a reduced subset of covariates and a dense 
regime where the converse was true.
Thus, pairs $(y_t, X_t)$ of response and predictors were simulated in a 
piece-wise stationary regimes. The dimensionality of the covariates was fixed at $p=20$, implying
that $X_t \in \mathbb{R}^{20}$. Changes occurred abruptly every 100 observations and two change-points were
considered, resulting in $300$ observations in total. 

Covariates, $X_t$, were simulated as described in Section \ref{sec--SimSettings}
within two alternating regimes; dense and sparse. 
The block-covariance structure remained fix within each regime (i.e., for 100 observations). 
Within the dense regime, 
a proportion $\rho_1= 0.8$ of 
regression coefficients were randomly selected and their values sampled from a standard Gaussian distribution. 
All remaining coefficients were set to zero. 
Similarly, in the case of the sparse regime, $\rho_2 = 0.2$ regression coefficients were randomly selected with remaining
coefficients set to zero. The regression coefficients remained fixed within each regime.

In order to benchmark the performance of the proposed RAP framework, 
streaming penalized Lasso models were also estimated using a fixed and stepwise constant sparsity parameters. 
As a result, the RAP algorithm was 
benchmarked against three distinct offline methods
for selecting the regularization parameter. 
In the case of a fixed sparsity parameter, 
$K=10$ fold cross-validation as well as Bayesian optimization were employed.
Finally, cross-validation was also employed to learn a stepwise constant regularization parameter.
This was achieved by performing cross-validation for the data within each regime. 
For each of these methods, 
their offline nature dictated that the 
entire dataset 
should be analyzed simultaneously (as opposed to in a streaming fashion by the RAP algorithm). 
As such, they serve to provide a benchmark but would infeasible in the context of streaming data.

Results for $N=500$ simulations are shown in Figure \ref{Sim2Fig}.
{\color{black}
	The 
	estimated time-varying regularization 
	parameter for both the linear and logistic regression models is shown on the left panels. 
	These results provide evidence that the RAP algorithm is able to 
	reliably track the piece-wise constant regularization parameters selected 
	by cross-validation. 
}
As expected, there is some lag directly after each change occurs, however, the 
estimated regression parameters is able to adapt thereafter. 
{\color{black}
	Figure \ref{Sim2Fig} also shows the mean negative log-likelihood over unseen samples, $C_{t+1}$.
}
We note there are abrupt spikes every 100 observations, corresponding to the 
abrupt changes in the underlying dependence structure. 
Detailed results are provided in Table \ref{ResultsTab}. We note that the proposed framework is able to
outperform 
the alternative offline approaches. In the case of the offline cross-validation and SMBO, 
this is to be expected as a fixed choice of regularization parameter is misspecified. 

\begin{figure}[!t]
	\includegraphics[width=0.8\textwidth]{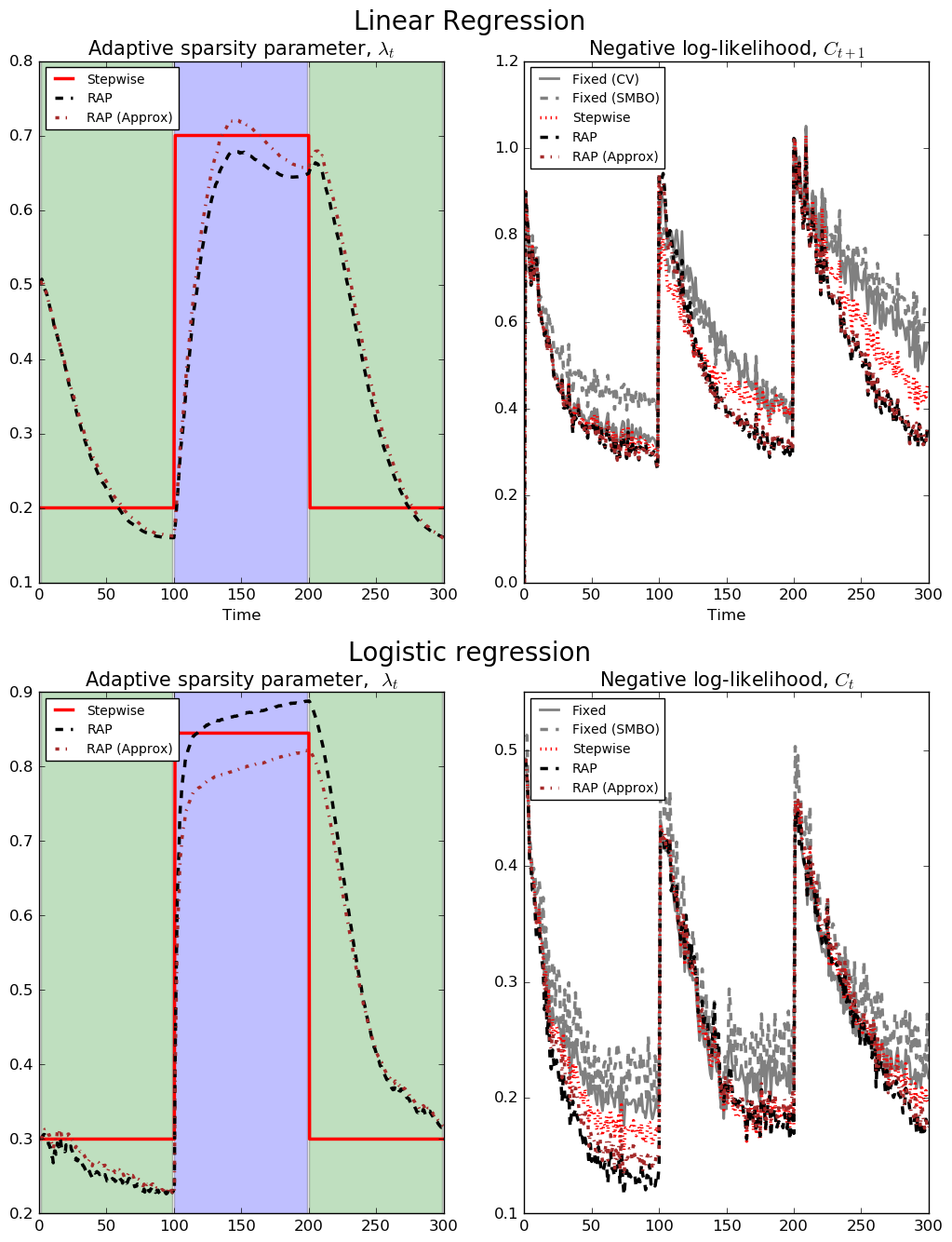}
	\centering
	\caption{
		{\color{black}
			Simulation results when estimating regularized streaming
			linear and logistic regression models. 
			Results for linear and logistic regression are shown across the 
			first and second rows respectively. 
			The left panels plot the 
			mean regularization parameter as estimated by the 
			RAP algorithm as well as the optimal piece-wise constant value selected via 
			cross-validation.
			The right panels plot the 
			mean negative log-likelihood, $C_{t+1}$, over time. We note that the 
			RAP algorithms outperform the offline alternatives.
		}
	} 
	\label{Sim2Fig}
\end{figure}

\begin{table}[h]
	\centering
	\caption{
		Detailed results for simulation involving non-stationary data over $N=500$ 
		independent iterations.  We report the 
		mean negative log-likelihood, $\bar C_t$, as well as the mean
		$F$-score, $\bar F_t$. 
		Standard errors are provided in brackets.
	}     
	\label{ResultsTab}
	\begin{tabular}{|c|c|c|c|c|}
		\hline
		&
		\multicolumn{2}{|c|}{\textbf{Linear regression}} &
		\multicolumn{2}{|c|}{\textbf{Logistic regression}} \\ \hline
		\\[-1em]
		{ \textbf{Algorithm}} & {\color{white}\large A}\hspace{-.4cm}  $\bar C_t$  &  $\bar F_t$ &  $\bar C_t$ &  $\bar F_t$  \\
		\hline
		Fixed (CV) & 0.58 (0.05) & 0.49 (0.05) & 0.25 (0.07) & 0.49 (0.04) \\
		\hline
		Fixed (SMBO) & 0.63 (0.05) & 0.50 (0.07) & 0.26 (0.08) & 0.49 (0.05)   \\
		\hline
		Stepwise & 0.51 (0.04) & 0.56 (0.04) & \textbf{0.21} (0.06)  & \textbf{0.53} (0.04) \\
		\hline
		RAP & \textbf{0.47} (0.04) & \textbf{0.64} (0.06) & \textbf{0.19} (0.04) & \textbf{0.58} (0.03) \\
		\hline
		RAP (Approx) & \textbf{0.48} (0.05) & \textbf{0.63} (0.07)  & \textbf{0.20} (0.04) & \textbf{0.55} (0.05) \\
		\hline
	\end{tabular}
\end{table}

\section{Application to fMRI data}

In this section we present an application of the RAP algorithm to 
task-based functional MRI (fMRI) data. 
This data corresponds to time-series measurements of blood oxygenation, a proxy for  neuronal activity, 
taken across a set 
of spatially remote brain regions. 
Our objective in this work is to quantify pairwise statistical dependencies across 
brain regions, 
typically referred to as functional connectivity within the neuroimaging
literature \citep{smith2011network}. 

While traditional analysis of functional connectivity was rooted on the assumption of stationarity,
there is growing evidence to suggest this is not the case \citep{hutchison2013dynamic}.
This particularly true in the context of task-based fMRI studies. 
Several methodologies have been proposed to address the non-stationary nature of 
fMRI data \citep{monti2014estimating}, many of 
which are premised on the use of penalized regression models such as 
those studied in this work. 
While such methods have made important progress in the study of 
non-stationary connectivity networks,
they have typically employed fixed regularization parameters.
This is difficult to justify in the context 
of non-stationary data and 
plausible biological justifications are not readily available.
The RAP algorithm is therefore
ideally suited to both accurately estimating non-stationary
connectivity structure as well as providing insight regarding whether the assumption of a 
fixed sparsity parameter is reasonable.


\subsection{Estimating connectivity via Lasso regressions}
\label{sec--neighbourhoodSel}
Estimating functional connectivity networks is fundamentally a statistical challenge.
A functional relationship is said to 
exist across two spatially remote brain regions 
if their corresponding time-series share some statistical dependence. 
While this can be quantified in a variety of ways, a 
popular approach is the use of Lasso regression models to infer the conditional independence structure 
of a particular node. In such an approach, the time-series of a given node 
is regressed against the time-series of all remaining nodes.
A functional relationship is subsequently inferred between the target node and 
all remaining nodes associated with a non-zero regression coefficient.
The connectivity structure across all nodes can then be inferred via a 
neighborhood selection approach
\citep{meinshausen2006high}.
The proposed RAP framework can directly be incorporated into such a model, resulting 
in time-varying conditional dependence structure where the underlying sparsity parameter is also inferred. 

%
%

\subsection{HCP Emotion Task Data}
Emotion task data from the Human Connectome
Project (HCP) 
was studied with 20 subjects selected at random.
During the task 
participants were presented with 
blocks of trials that either required them to decide which of two 
faces presented on the bottom of the screen match the face 
at the top of the screen, or which of two shapes 
presented at the bottom of the screen match the shape at the top of the screen.  
The faces had either an angry or fearful expression while the shapes 
represented the emotionally neutral condition.
Twenty regions were selected from an initial subset of 84 brain regions based on the Desikan-Killiany atlas. 
%
Data for each subject therefore consisted of
$n=175$ observations across $p=20$ nodes.

\begin{figure}[!t]
	\includegraphics[width=.87\textwidth]{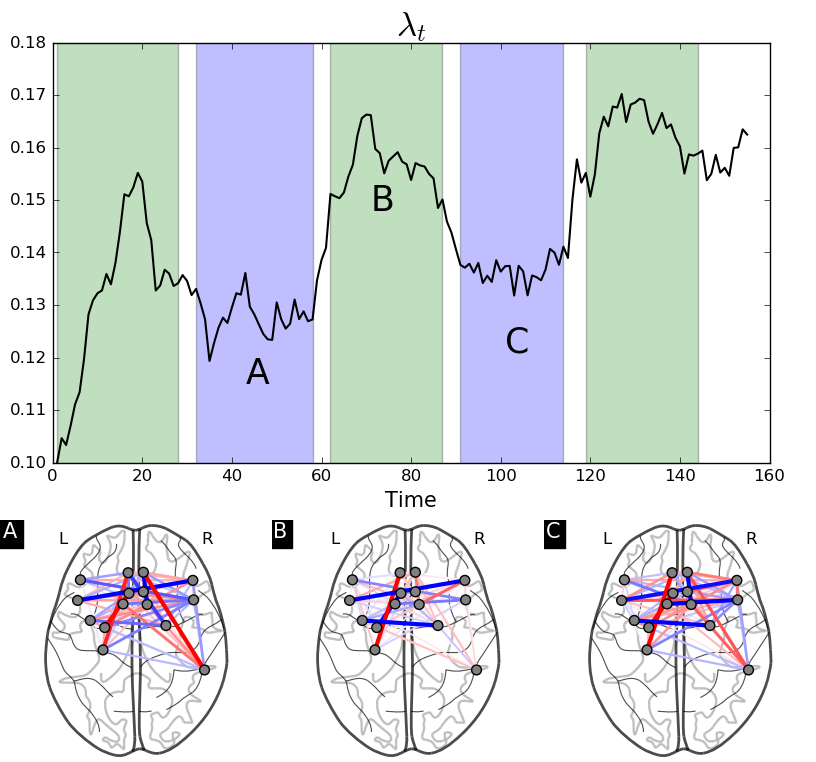}
	\centering
	\caption{\textit{Top}: the mean sparsity parameter is shown as a function of time. The background color indicates 
		the nature of the task at hand (green indicates neutral task while blue indicates the emotion task).
		\textit{Bottom}: estimated networks visualizing the estimated connectivity structure at 
		three distinct points in time. Edge colors indicate the 
		nature of the dependence (blue indicates a positive dependence, red a negative dependence). } 
	\label{AppFig}
\end{figure}

\subsection{Results}

Data for each subject was analyzed independently where the time-varying estimates of the 
conditional dependence structure for each node were estimated as 
described in Section \ref{sec--neighbourhoodSel}.
A fixed forgetting factor 
of $r=.95$ was employed throughout with a stepsize parameter $\epsilon=.025$.
The exact gradient was employed when updating the sparsity parameter at each iteration.

The mean sparsity parameter over all subjects is shown in  the top panel of 
Figure \ref{AppFig}. 
We observe decreased sparsity parameters for blocks
in which subjects were presented with emotional 
(i.e., angry or fearful) faces (top panel, purple shaded areas)
as compared to blocks in which subjects were shown neutral shapes (top panel, green shaded areas). 
The oscillation in sparsity parameter is highly correlated with task onset.
When inspecting the networks estimated using the time
varying sparsity parameter (bottom panel), we find strong
coupling amongst many of the regions during the emotion processing
blocks (A and C) compared to a clearly sparser
network representation for blocks that require no emotion
processing (i.e., neutral shapes, block B). This is to be expected
as the selected regions are core hubs involved with
emotion processing; therefore explaining the higher network
activity during the emotion task. 

\section{Conclusion}

In this work we have presented a  framework through which to 
learn  time-varying regularization parameters in the context of streaming generalized
linear models.
An approximate algorithm is also provided to 
address issues concerning computational efficiency.
%
%
We present two simulation studies which demonstrate the capabilities of the RAP
framework.
These simulations show that the proposed  framework
is capable of tracking the regularization parameter both in  a
stationary as well as non-stationary setting.
Finally, we present an application to task-based fMRI data, which is widely accepted to be 
non-stationary \citep{hutchison2013dynamic}.

Future work will involve extending the RAP framework to consider 
alternative regularization schemes.
In particular an $\ell_2$ penalty 
could also be incorporated as the derivative, $\frac{\partial \hat \beta_t (\lambda )}{\partial \lambda}$, 
is  available in closed form.

Finally, the methods presented in this manuscript have been motivated 
by the study of fMRI data in real-time \citep{monti2016real}. Future work will look to 
extend the proposed methodology, 
for example by combining with current approaches which involve
graph embeddings \citep{monti2017decoding} or 
novel applications of real-time fMRI such as those 
described by \cite{lorenz2016automatic}  and 
\cite{lorenz2017neuroadaptive}. Another exciting avenue 
would be to use the proposed methods to understand 
variability in dynamic functional connectivity \citep{monti2015learning}. 
Furthermore, 
it would also be interesting to consider alternative  
applications such as 
cyber-security \citep{gibberd2016time}, gene expression 
data \citep{gibberd2017regularized}  and finance.


%

\bibliographystyle{plainnat}
\bibliography{sigproc.bib}
\begin{appendices}
	\section{Proof of Proposition \ref{prop2}}
	\label{AppendixA}
	For a given regularization parameter, $\lambda$, the corresponding 
	vector of estimated regression coefficients, $\hat \beta (\lambda)$,  can be computed
	by minimizing the non-smooth objective, $L_t(\beta, \lambda)$,  provided in equation (\ref{l1Objective}). 
	The sub-gradient is defined as:
	\begin{eqnarray}
	\nabla_{\beta} L_t (\beta, \lambda) = -X_{1:t}^T W (y_{1:t} - \bm{\mu} ) \frac{\partial \bm{\eta}}{\partial \bm{\mu}} + \lambda ~ \mbox{sign} ( \beta) 
	\end{eqnarray}
	We write $\bm{\mu} $ to denote the vector of predicted means, $\mu_i = g^{-1}( \eta_i) = g^{-1} ( X_i^T \beta)$ and 
	$\frac{\partial \bm{\eta}}{\partial \bm{\mu}}$ to denote a vector with entries 
	$\frac{\partial {\eta}_i}{\partial {\mu}_i}$.
	Note that in the case of normal linear regression we have that $\mu_i = \eta_i = X_i^T \beta$ and we therefore recover equation (\ref{subGradEq}). 
	
	As in Proposition \ref{Prop1}, we have that for any choice of regularization parameter,
	the sub-gradient evaluated at $\hat \beta_t (\lambda)$ must 
	satisfy:
	$$ \left . \nabla_{\beta} L_t (\beta, \lambda) \right |_{\beta = \hat \beta_t( \lambda) } \ni 0.  $$
	We therefore compute the derivative with respect to $\lambda$ in order to obtain:
	\begin{align}
	\frac{\partial}{ \partial \lambda} \left (  \left . \nabla_{\beta} L_t (\beta, \lambda) \right |_{\beta  = \hat \beta_t( \lambda) } \right ) &= 0 \\
	&= \frac{\partial}{ \partial \lambda} \left (  -X_{1:t}^T W (y_{1:t} - \bm{\mu }) \frac{\partial \bm{\eta}}{\partial \bm{\mu}}  \right ) + \mbox{sign} (\hat \beta_t( \lambda)  )  \\
	&= X_{1:t}^T W \frac{\partial \bm{\mu }}{ \partial \lambda } \frac{\partial \bm \eta}{\partial \bm \mu } + \mbox{sign} (\hat \beta_t( \lambda)  ) \\
	&= X_{1:t}^T W  X_{1:t} \frac{\partial \hat \beta_t( \lambda) }{ \partial \lambda }  + \mbox{sign} ( \hat \beta_t( \lambda)  )
	\label{FinalStep} 
	\end{align}
	where equation (\ref{FinalStep}) follows from the fact that:
	$$ \frac{\partial \bm{\mu }}{ \partial \lambda } = \frac{\partial \bm{\mu }}{ \partial \bm{\eta} }
	\frac{\partial \bm{\eta}}{ \partial \hat \beta_t(\lambda) }
	\frac{ \partial \hat \beta_t(\lambda) }{\partial \lambda }
	= \frac{\partial \bm{\mu }}{\partial \bm{\eta}} X_{1:t} \frac{ \partial \hat \beta_t(\lambda) }{\partial \lambda }.
	$$
	Rearranging equation (\ref{FinalStep}) yields the result.

\end{appendices}

\end{document}